\documentclass[11pt,a4paper]{scrartcl}
\usepackage{natbib}
\usepackage{graphicx}
\usepackage{subfig}
\usepackage{amsfonts}
\usepackage{amsthm}
\usepackage{amssymb}
\usepackage{amstext}
\usepackage{amsmath}
\usepackage{amsthm}
\usepackage{pstricks}
\usepackage{pspicture}
\usepackage{amscd}
\usepackage{appendix}
\usepackage{arydshln}
\usepackage{tipa}
\usepackage{pifont}
\usepackage[margin=0.7in]{geometry}
\newcommand{\cmark}{\ding{51}}%
\newcommand{\xmark}{\ding{55}}%

 \newtheorem{theorem}{Theorem}
 \newtheorem{proposition}{Proposition}
  
 \newtheorem{definition}{Definition}

	\newtheorem{property}{Property}
  \theoremstyle{definition}
  \newtheorem{example}{Example}

\newcommand{\sgn}{\textrm{sgn}}
\begin{document}
\title{Studying a set of properties of inconsistency indices for pairwise comparisons\footnote{This is a preprint of the paper: Brunelli M., Studying a set of properties of inconsistency indices for pairwise comparisons, \emph{Annals of Operations Research}, doi:10.1007/s10479-016-2166-8}}
\author{
{Matteo Brunelli}
\\
{\normalsize  Systems Analysis Laboratory, Department of Mathematics and Systems Analysis} \\
{\normalsize Aalto University}, {\normalsize P.O. Box 11100, FI-00076 Aalto, Finland}
\\ {\normalsize e--mail:
\texttt{matteo.brunelli@aalto.fi}}
}
\date{}

\maketitle \thispagestyle{empty}


\begin{center}
{Abstract}
\end{center}

{\small \noindent Pairwise comparisons between alternatives are a well-established tool to decompose decision problems into smaller and more easily tractable sub-problems. However, due to our limited rationality, the subjective preferences expressed by decision makers over pairs of alternatives can hardly ever be consistent. Therefore, several inconsistency indices have been proposed in the literature to quantify the extent of the deviation from complete consistency. Only recently, a set of properties has been proposed to define a family of functions representing inconsistency indices. The scope of this paper is twofold. Firstly, it expands the set of properties by adding and justifying a new one. Secondly, it continues the study of inconsistency indices to check whether or not they satisfy the above mentioned properties. Out of the four indices considered in this paper, in its present form, two fail to satisfy some properties. An adjusted version of one index is proposed so that it fulfills them.}

 \vspace{0.3cm}
 \vspace{0.3cm}


In decision making problems it is often common practice to use pairwise comparisons between alternatives as a basis to assign scores to the same alternatives. Pairwise comparisons allow the decision maker to decompose the problem of assigning scores to alternatives into smaller problems, where only two alternatives are considered at a time.

Another reason for using pairwise comparisons is that their use allows an estimation of the inconsistency of the preferences of a decision maker. In the literature, consistency of preferences is commonly related with the rationality of a decision maker and his ability in discriminating between alternatives \citep{Irwin1958}. Consider, for sake of illustration, three stones (alternatives) $x_{1},x_{2},x_{3}$. If, for instance, $x_{1}$ is reputed twice as heavy as $x_{2}$, and $x_{2}$ twice as heavy as $x_{3}$, then it is reasonable to assume that $x_{1}$ should be four times as heavy as $x_{3}$. This situation is called consistent, as the pairwise comparisons of the decision maker respect a principle of transitivity/rationality, and is depicted in Figure \ref{fig:consistency}(a). An example of inconsistent pairwise comparisons is illustrated in Figure \ref{fig:consistency}(b).
\begin{figure}[htb]
\centering
\subfloat[Consistent triad of pairwise comparisons.]{
\includegraphics[scale=0.37]{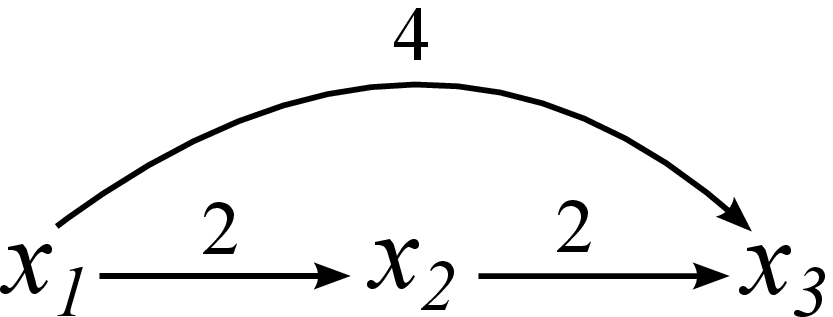} 
}
\hspace{1cm}
\subfloat[Inconsistent triad of pairwise comparisons.]{
\includegraphics[scale=0.37]{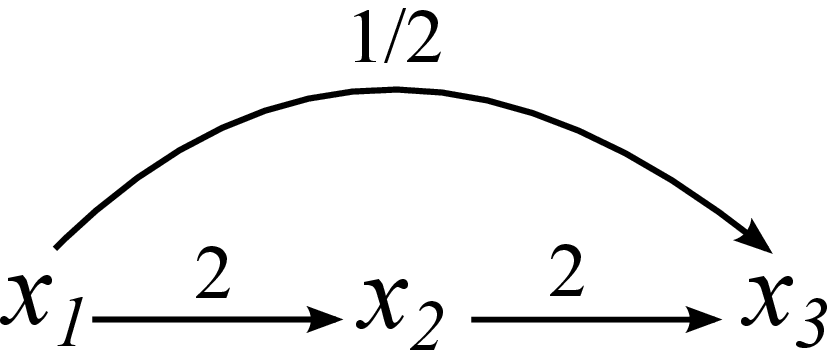}
}
\caption{Example of consistent and inconsistent triads of pairwise comparisons on $x_{1},x_{2},x_{3}$.\label{fig:consistency}}
\end{figure}

There is a meeting of minds on accepting preferences which are not consistent, but not too inconsistent either. In this paper, with the term \emph{inconsistency} we mean a deviation from the condition of full consistency. In the theory of the AHP, \citet{Saaty1993,Saaty2013} required pairwise comparisons to be near consistent, i.e. not too inconsistent. As recalled by \citet{Gass2005}, \citet{LuceRaiffa1957} shared the same opinion in accepting inconsistencies and wrote ``No matter how intransitivities arise, we must recognize that they exist, and we can take a little comfort in the thought that they are an anathema to most of what constitutes theory in the behavioral sciences today''. On a similar note, \citet{Fishburn1999} wrote that ``Transitivity is obviously a great practical convenience and a nice thing to have for mathematical purposes, but long ago this author ceased to understand why it should be a cornerstone of normative decision theory''.

It is in this context---where consistency is an auspicable but hardly ever achievable condition---that it becomes crucial to quantify inconsistency. Such quantification is indeed possible, since it is natural to envision that the notion of inconsistency is a matter or degree.
Consequently, a wealth of inconsistency indices has been proposed in the literature; for instance the Consistency Index \citep{Saaty2013}, the Harmonic Consistency Index \citep{SteinMizzi2007}, the Geometric Consistency Index \citep{AguaronMoreno2003}, the statistical index by \citet{LinEtAl2013}, and the index by \citet{Kulakowski2015}, just to cite few.

It is worth noting that the study of inconsistency of preferences is not limited to the single mathematical methods employing pairwise comparisons, as for instance the AHP. It is the case to remark that the study of inconsistency is immune from many of the criticisms moved against specific mathematical methods employing them. For instance, one of the critical points of the Analytic Hierarchy Process (AHP) is the rank reversal, which was discovered by \citet{BeltonGear1983} and recently surveyed by \citet{MalekiZahir2013}. Similarly, already \citet{WatsonFreeling1982,WatsonFreeling1983} questioned the interpretation of the weights in the AHP and their use in the aggregation of different priority vectors. In part, also the criticisms by \citet{Dyer1990a,Dyer1990b} were triggered by the interpretation of the weights. Nevertheless, even though the above mentioned criticisms are to be taken into account, they are connected with the aggregation and interpretation of priority vectors proposed for the AHP, and therefore they will not affect the subject matter of inconsistency evaluation. Further support to the use of pairwise comparison matrices and their interpretation comes from the fact that pairwise comparison matrices as defined in this paper are group isomorphic \citep{CavalloDApuzzo2009}---and thus structurally identical---to the probabilistic preference relations studied by \citet{LuceSuppes1965}. Such a strict connection between these two representations of valued preferences does not only make them mutually supportive, but increases the relevance of studying one of them---as it is going to be done in this paper---since abstract results are then extendible to the other one.

The use of the notion of inconsistency has gone beyond its mere quantification. One prominent use of inconsistency indices is that of localizing the inconsistency and detect what comparisons are the most contradictory \citep{ErguEtAl2011} and guide the decision maker when he tries to obtain sufficiently consistent preferences \citep{PereiraCosta2015}. This process was also advocated by \citet{Fishburn1968} in a discussion on decision theory: ``If the individual's preferences appear to violate a ``rational'' preference assumption, the theory suggests that he reexamine and revise one or more preference judgments to eliminate the inconsistency.''.
Another use of inconsistency indices regards pairwise comparison matrices with missing entries. In these situations, inconsistency indices have been used as objective functions to be minimized to find the most plausible values of the missing comparisons with respect to the elicited ones \citep{Koczkodaj1999,LamataPelaez2002,ShiraishiEtAl1999,ChenEtAl}. All this can be seen as evidence on the role played by inconsistency indices in the decision process, and consequently on the importance of having realiable indices.

Inconsistency of preferences has been studied empirically \citep{BozokiEtAl2013}, and existing studies on inconsistency indices compared them numerically \citep{BrunelliCanalFedrizzi} and showed that some indices are very different and therefore can lead to very different evaluations of the inconsistency of preferences. Conversely, it was proven that some of them are in fact proportional to each other \citep{BrunelliCritchFedrizzi2013}. Recently, \citet{BrunelliFedrizziAxioms} and \citet{KoczkodajSzwarc2014} proposed two formal approaches. \citet{BrunelliFedrizziAxioms} proposed five properties in the form of axioms to formalize the concept of inconsistency index and then tested on some well-known indices.

In the pursuit of a formal treatment of inconsistency quantification, this paper presents some developments concerning the aforementioned set of properties. Firstly, in Section \ref{sec:3}, a new property, of invariance under inversion of preferences, is introduced and its role is discussed. Secondly, Section \ref{sec:4} contains further results on the satisfaction of the properties by some known inconsistency indices. More specifically, we shall study four indices and discover that, in its present form, two do not fully satisfy the set of properties. An adjustment of one index is then proposed so that it satisfies them. Finally, Section \ref{sec:5} offers a concise discussion on the role of inconsistency quantification and on the results obtained in this paper.

\section{Pairwise comparison matrices and inconsistency indices}
\label{sec:2}
Given a set $X=\{ x_{1},\ldots,x_{n} \}$ of $n$ alternatives, a \emph{pairwise comparison matrix} is a positive square matrix $\mathbf{A}=(a_{ij})_{n \times n}$ such that $a_{ij}a_{ji}=1$, where $a_{ij}>0$ is the subjective assessment of the relative importance of the $i$th alternative with respect to the $j$th one. 
A pairwise comparison matrix can be seen as a convenient mathematical structure into which valued pairwise comparisons between alternatives are collected. Its general and its simplified (thanks to $a_{ij}a_{ji}=1$) forms are the following,
\begin{equation*}
\label{id:matricecanonica} \mathbf{A}=
(a_{ij})_{n \times n} =
\begin{pmatrix}
a_{11} & a_{12} & \ldots & a_{1n} \\
a_{21} & a_{22} & \ldots & a_{2n} \\
\vdots & \vdots & \ddots & \vdots \\
a_{n1} & a_{n2} & \ldots & a_{nn}
\end{pmatrix}
= \begin{pmatrix}
1                & a_{12}           & \cdots           & a_{1n} \\
\frac{1}{a_{12}} & 1                & \cdots           & a_{2n} \\
\vdots           & \vdots           & \ddots           & \vdots \\
\frac{1}{a_{1n}} & \frac{1}{a_{2n}} & \cdots           &1
\end{pmatrix}.
\end{equation*}
The rest of the paper will follow the usual interpretation of entries $a_{ij}$ in terms of ratios between quantities expressible on a ratio scale with a zero element. The classical example is that of $x_{1}$ and $x_{2}$ being stones and $a_{ij}$ being the numerical estimation of the ratio between their weights. Note that this approach considers entries $a_{ij}>0$ taking values from an unbounded scale and complies with the formal treatment given by \citet{HermanKoczkodaj1996} and \citet{KoczkodajSzwarc2014}. 
Furthermore, with this interpretation, a pairwise comparison matrix is \emph{consistent} if and only if
\begin{equation}
\label{eq:consistency}
a_{ik}=a_{ij}a_{jk}~~~\forall i,j,k,
\end{equation}
which means that each direct comparison $a_{ik}$ is exactly backed up by all indirect comparisons $a_{ij}a_{jk}~\forall j$. 
%
For notational convenience, the set of all pairwise comparison matrices is defined as
\[
\mathcal{A}= \left\{ \mathbf{A}=(a_{ij})_{n \times n} | a_{ij}>0, a_{ij}a_{ji}=1 ~\forall i,j, ~ n>2 \right\}.
\]
The set of all \emph{consistent} pairwise comparison matrices $\mathcal{A}^{*} \subset  \mathcal{A}$ is defined accordingly,
\[
\mathcal{A}^{*}= \{ \mathbf{A}=(a_{ij})_{n \times n} | \mathbf{A} \in \mathcal{A}, a_{ik} = a_{ij}a_{jk} ~ \forall i,j,k \}.
\]
An inconsistency index is a function $I:\mathcal{A} \rightarrow \mathbb{R}$ which evaluates the intensity of deviation of a pairwise comparison matrix $\mathbf{A}$ from its consistent form (\ref{eq:consistency}). In other words, the value $I(\mathbf{A})$ is an estimation of how much irrational the preferences collected in $\mathbf{A}$ are. Up to now, various inconsistency indices have been introduced heuristically, and an open question relates to what set of properties should be used to characterize them. That is, all the reasonable properties for a function $I$ to fairly capture inconsistency could be used for various purposes; for example to check the validity of already proposed indices \citep{BrunelliFedrizziAxioms}, devise new ones, and derive further properties \citep{BrunelliFedrizzi2014}. Figure \ref{fig:axioms} offers a snapshot of the meaning of the set of properties.
\begin{figure}[htbp]
	\centering
		\includegraphics[width=0.37\textwidth]{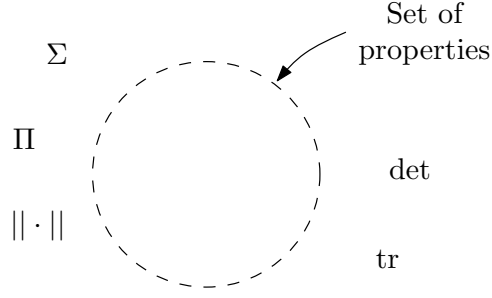}
	\caption{The set of properties can be used to define a family of functions which can be used to estimate inconsistency, and discards functions which do not make sense if used as inconsistency indices, e.g. the trace and the determinant of $\mathbf{A}$.}
	\label{fig:axioms}
\end{figure}

\noindent  \citet{BrunelliFedrizziAxioms} proposed five properties to characterize inconsistency indices. Since these properties were already justified and defined in the original work, they are here only briefly recalled. Note that they were organized in the form of an axiomatic systems, meaning that the soundness of single properties implies the soundness of the entire set of properties, i.e. the ``logical intersection'' of the properties.
\begin{description}
	\item[P1:] There exists a unique $\nu \in \mathbb{R}$ representing the situation of full consistency, i.e.
	\[
	 \label{P1}
\exists ! \nu \in \mathbb{R} \text{ such that } I(\mathbf{A})= \nu \Leftrightarrow \mathbf{A} \in \mathcal{A}^{*}.
	\]
	\item[P2:] Changing the order of the alternatives does not affect the inconsistency of preferences. That is,
	\begin{equation*}
\label{eq:permutation}
I(\mathbf{P}\mathbf{A}\mathbf{P}^{T})=I(\mathbf{A}),
\end{equation*}
for any permutation matrix $\mathbf{P}$.

\item[P3:] If preferences in $\mathbf{A}$ are intensified, then the inconsistency cannot decrease. More formally, since the power is the only meaningful function to intensify preferences, we defined $\mathbf{A}(b)=\left( a_{ij}^{b} \right)_{n \times n}$. Then, the property is as follows,
\[
I(\mathbf{A}(b)) \geq I(\mathbf{A}) ~~~~\forall \mathbf{A} \in \mathcal{A}, ~~b\geq 1.
\]

\item[P4:] Given a consistent pairwise comparison matrix and considering an arbitrary non-diagonal element $a_{pq}$ (and its reciprocal $a_{qp}$) such that $a_{pq} \neq 1$, then, as we push its value far from its original one, the inconsistency of the matrix should not decrease. More formally, given a consistent matrix
$\mathbf{A} \in \mathcal{A}^{*}$, let $\mathbf{A}_{pq}(\delta)$ be the
inconsistent matrix obtained from \textbf{A} by replacing the
entry $a_{pq}$ with $a_{pq}^{\delta}$, where $\delta \neq 1$.
Necessarily, $a_{qp}$ must be replaced by $a_{qp}^{\delta}$ in
order to preserve reciprocity. Let $\mathbf{A}_{pq}(\delta')$ be the
inconsistent matrix obtained from \textbf{A} by replacing entries $a_{pq}$ and $a_{qp}$ with $a_{pq}^{\delta'}$ and $a_{qp}^{\delta'}$ respectively.
The property can then be formulated as
\begin{equation}
\label{monotonicity}
\begin{split}
 \delta' > \delta > 1 & \Rightarrow I(\mathbf{A}_{pq}(\delta')) \geq I(\mathbf{A}_{pq}(\delta)) \\
 \delta' < \delta < 1 & \Rightarrow I(\mathbf{A}_{pq}(\delta')) \geq I(\mathbf{A}_{pq}(\delta)),
\end{split}
\end{equation}
for all $\delta \neq 1$, $p,q=1,\ldots,n$, and $\mathbf{A} \in \mathcal{A}^{*}$.

\item[P5:] Function $I$ is continuous with respect to the entries of $\mathbf{A}$.
\end{description}

\section{A new property of invariance under inversion of preferences}
\label{sec:3}

Preferences expressed in the form of a pairwise comparison matrix $\mathbf{A}$ can be inverted by taking its transpose $\mathbf{A}^{T}$. For instance, if $a_{ij}=2$ in $\mathbf{A}$ is inverted into $a_{ij}=1/2$ we have that the intensity of preference is the same, but the direction is inverted. Clearly, by inverting all the preferences we change their polarity, but leave their structure unchanged. Thus, it is reasonable to expect a structural property of preferences---as inconsistency is---to be invariant under inversion. This can be formalized in the following property of invariance under inversion of preferences (P6).

\begin{property}[P6]
An inconsistency index satisfies P6, if and only if $I(\mathbf{A})=I(\mathbf{A}^{T})~\forall \mathbf{A} \in \mathcal{A}$.
\end{property}

The previous justification of this property can be transposed into an example. Consider the following matrix $\mathbf{A}$ and its transpose $\mathbf{A}^{T}$.
\begin{equation}
\label{eq:ex_trans}
\mathbf{A}=
\begin{pmatrix}
1 & 1/2 & 1/4 \\
2 & 1   & 1/3 \\
4 & 3   & 1
\end{pmatrix}
\hspace{0.7cm}
\mathbf{A}^{T}=
\begin{pmatrix}
1 & 2 & 4 \\
1/2 & 1   & 3 \\
1/4 & 1/3   & 1
\end{pmatrix}
\end{equation}
One can equivalently express the structure of the preferences by means of directed weighted graphs with nodes $x_{i}$ and values of the edges $a_{ij}$. Figure \ref{fig:A6} represents these graphs for $\mathbf{A}$ and $\mathbf{A}^{T}$, respectively.
\begin{figure}[htb]
\centering
\subfloat[Graph of $\mathbf{A}$.]{
\includegraphics[scale=1.2]{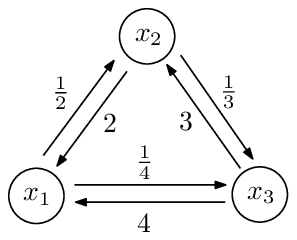} 
}
\hspace{0.7cm}
\subfloat[Graph of $\mathbf{A}^{T}$.]{
\includegraphics[scale=1.2]{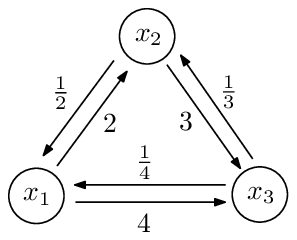}
}
\caption{Graphs of $\mathbf{A}$ and $\mathbf{A}^{T}$.\label{fig:A6}}
\end{figure}
 
\noindent The two graphs are identical, with the only exception of the directions of the arrows. Now, if we do not impose P6, we might end up with inconsistency indices which consider the violation of the condition of consistency in the direction $1 \rightarrow 2 \rightarrow 3 \rightarrow 1$ more or less (but not equally) important than the violation in the direction $1 \rightarrow 3 \rightarrow 2 \rightarrow 1$, although both directions equivalently reflect the same structure of preferences.

The justification of P6 comes from theoretical intuition, but in some decision processes both a matrix $\mathbf{A}$ ad its transpose $\mathbf{A}^{T}$ can actually appear. Examples are applications in group decision making where pairwise comparison matrices are used in surveys on customers' needs  and users' satisfactions, as for instance done by \citet{NikouMezei2013} and \citet{NikouMezeiSarlin2015}. In these contexts there are as many pairwise comparison matrices as responding customers (usually a large number) and therefore it is not completely unlikely to find both preferences represented by $\mathbf{A}$ and $\mathbf{A}^{T}$, especially when there are few alternatives and intensities of preference are weak. i.e., values of entries are close to 1.

Note that, in general, an inversion of preferences cannot be obtained by row-column permutations. For example, given the matrices $\mathbf{A}$ and $\mathbf{A}^{T}$ in (\ref{eq:ex_trans}), there does \emph{not} exist a permutation matrix $\mathbf{P}$ such that $\mathbf{P}\mathbf{A}\mathbf{P}^{T}=\mathbf{A}^{T}$.

One natural question is whether or not this new property, P6, is implied by a conjoint application of the others (independence) and if, when added to the set P1--P5, does not make it contradictory (logical consistency). The following theorem claims the independence and the logical consistency of the properties P1--P6. 

\begin{theorem}
\label{th:A6}
Properties P1--P6 are independent and form a logically consistent axiomatic system.
\end{theorem}

\begin{proof}
See Appendix.
\end{proof}

In light of the previously offered justification and Theorem \ref{th:A6}, one concludes that P6
is another interesting property of inconsistency indices, and is independent from P1--P5.

\section{Extending the analysis of the satisfaction of the axioms}
\label{sec:4}

Previous research \citep{BrunelliFedrizziAxioms,CavalloDApuzzoIPMU} has made the effort of proving whether or not some known inconsistency indices satisfy the set of properties. This section continues the investigation on the satisfaction of the set of properties by testing four indices proposed in the literature and used in real-world decision making problems. For each index we shall recall the definition and highlight its relevance in both theory and practice.

\subsection{Index $K$ by Koczkodaj}

The following index, $K$, was introduced by \citet{Koczkodaj1993} and extended by \citet{Koczkodaj1994}.

\begin{definition}[Index $K$ \citep{Koczkodaj1994}]
\label{def:K}
Given a pairwise comparison matrix $\mathbf{A}$, the index $K$ is
\begin{equation}
\label{eq:K}
K(\mathbf{A})=\max \left\{ \min \left\{
\left| 1- \frac{a_{ik}}{a_{ij}a_{jk}} \right|, \,
\left| 1- \frac{a_{ij}a_{jk}}{a_{ik}} \right|
\right\} : 1 \leq i < j < k \leq n \right\}.
\end{equation}
\end{definition}
This index has been used to estimate missing entries of incomplete pairwise comparisons \citep{Koczkodaj1999} and in real-world applications in problems such as the evaluation of research institutions in Poland \citep{KoczkodajEtAl2014} and medical diagnosis \citep{KakiashviliEtAl2012}. It was also compared to Saaty's Consistency Index \citep{BozokikRapcsak2008} and on occasions even claimed superior to it \citep{KoczkodajSzwarc2014}. Given its theoretical and practical relevance, it is therefore important to check what properties it satisfies. Here we show that index $K$ satisfies the six properties P1--P6.

\begin{proposition}
\label{prop:K}
Index $K$ satisfies the properties P1--P6.
\end{proposition}

\begin{proof}
It is straightforward, and thus omitted, to show that properties P1, P2, P5, and P6 are satisfied. For P3 we need to show that the local inconsistency for the generic transitivity $(i,j,k)$, 
\begin{equation}
\label{eq:trans_K}
\min \left\{ \left| 1- \frac{a_{ik}^{b}}{a_{ij}^{b}a_{jk}^{b}} \right|, \left| 1- \frac{a_{ij}^{b}a_{jk}^{b}}{a_{ik}^{b}} \right| \right\} \, ,
\end{equation}
is non-decreasing for $b \geq 1$. We can do it by proving that $\frac{\partial K}{\partial b} \geq 0 ~\forall b>1$. With $x^{b}:=\frac{a_{ik}^{b}}{a_{ij}^{b}a_{jk}^{b}}$, we study the two quantities
\[
\text{I}=|1-x^{b}|~~~~\text{II}=|1-x^{-b}|.
\]
If the triple $(i,j,k)$ is consistent, then $x=1$ and P3 is satisfied. If the triple $(i,j,k)$ is not consistent, then $x \neq 1$ and positive, and the derivatives of $\text{I}$ and $\text{II}$ in $b$ are:
\begin{align*}
\frac{\partial \text{I}}{\partial b} &= -x^{b} \log(x)  \sgn \left( 1-x^b \right) \\
\frac{\partial \text{II}}{\partial b} &= x^{-b} \log(x)  \sgn \left( 1-x^{-b} \right).
\end{align*}
Given $b \geq 1$, if $x \neq 1$, then $\frac{\partial \text{I}}{\partial b}$ and $\frac{\partial \text{II}}{\partial b}$ are positive, which proves that (\ref{eq:trans_K}) is a non-decreasing function for $b \geq 1$. It follows that also $K$ is a non-decreasing function of $b \geq 1$.\\
To prove the satisfaction of P4 we start considering
\begin{equation}
\label{eq:K_A3}
 \min \left\{ \left| 1- \frac{a_{ik}^{\delta}}{a_{ij}a_{jk}} \right|, \left| 1- \frac{a_{ij}a_{jk}}{a_{ik}^{\delta}} \right| \right\}
\end{equation}
with $a_{ik}=a_{ij}a_{jk}$. By setting $y=a_{ik}=a_{ij}a_{jk}$ we can rewrite it as 
\[
 \min \left\{ \left| 1-y^{\delta-1} \right|, \left|  1-y^{1-\delta} \right| \right\}
\]
and show that it is a non-decreasing function for $b \geq 1$ and a non-increasing function for $0<b \leq 1$. We then need to study the following quantities:
\[
\text{I}=|1-y^{\delta -1}|~~~~~~\text{II}=|1-y^{1-\delta}|
\]
and their derivatives in $\delta$
\[
\frac{\partial \text{I}}{\partial \delta}=-y^{\delta -1} \log (y) \sgn \left( 1-y^{\delta -1} \right)~~~~~~\frac{\partial \text{II}}{\partial \delta}= y^{1 - \delta } \log (y) \sgn \left( 1-y^{1 - \delta} \right).
\]
By studying their sign we can derive that
\begin{align*}
0<\delta < 1 & \Rightarrow \frac{\partial \text{I}}{\partial \delta}, \frac{\partial \text{II}}{\partial \delta} \leq 0 \Rightarrow \frac{\partial K}{\partial \delta} <0 \\
\delta > 1 & \Rightarrow \frac{\partial \text{I}}{\partial \delta}, \frac{\partial \text{II}}{\partial \delta} \geq 0 \Rightarrow \frac{\partial K}{\partial \delta} >0.
\end{align*}
Similarly, P4 can be proven also in the case when the exponent $\delta$ is at the denominator of $\frac{a_{ik}}{a_{ij}a_{jk}}$ in (\ref{eq:K_A3}). 
\end{proof}

\subsection{Index $AI$ by Salo and H\"{a}m\"{a}l\"{a}inen}

\citet{SaloHamalainen1995,SaloHamalainen1997} proposed their inconsistency index, $AI$, which stands for ambiguity index. Their inconsistency index has been implemented in the online decision making platform Web-HIPRE \citep{Mustajoki2000} and has been used, for instance, in the analysis of a real-world governmental decision on energy production alternatives \citep{SaloHamalainen1995} and in traffic planning \citep{HamalainenPoyhonen1996}.

\begin{definition}
\label{def:AI}
Given a pairwise comparison matrix $\mathbf{A}$ and an auxiliary matrix $\mathbf{R}=(r_{ij})_{n \times n}$ with $r_{ij}=\{ a_{ik}a_{kj} | k=1,\ldots,n \}$, then the index $AI$ is
\begin{equation}
\label{eq:AI}
AI(\mathbf{A})=\frac{2}{n(n-1)} \sum_{i=1}^{n-1} \sum_{j=i+1}^{n} \frac{\max(r_{ij})-\min(r_{ij})}{(1+\max(r_{ij}))(1+\min(r_{ij}))}.
\end{equation}
\end{definition}

The interpretation of $AI$ is original and different from those of other indices. Consider that $r_{ij}$ is \emph{not} a real number but, instead, the set of possible values of $a_{ij}$ as could be deduced from indirect comparisons $a_{ik}a_{kj} \, \forall k$.\\
For example, given the matrix
\[
\mathbf{A}=
\begin{pmatrix}
1     &     2    &   3     & 1/2 \\
1/2   &     1    &   4     & 1/3 \\
1/3   &     1/4  &   1     & 2   \\
2     &     3    &   1/2   & 1
\end{pmatrix},
\]
we have
\[
r_{14}= \{ a_{11}a_{14}, a_{12}a_{24}, a_{13}a_{34}, a_{14}a_{44}   \} = \left\{ \frac{1}{2},\frac{2}{3},6 \right\},
\]
from which we obtain $\max(r_{14})=6$ and $\min(r_{14})=1/2$.\\
It is possible to build an interval-valued matrix
\[
\bar{\mathbf{A}}= \left( \bar{a}_{ij} \right)_{n \times n}=\left( [ \min(r_{ij} ) , \max(r_{ij})] \right)_{n \times n}
\]
such that the `true value' of the comparison between $x_{i}$ and $x_{j}$ shall lie in the interval $\bar{a}_{ij}$. The larger the intervals are, the more inconsistent the matrix, and in fact $AI$ is a normalized sum of the lengths of the intervals $\bar{a}_{ij}$. The following shows that $AI$ satisfies all properties except P3.

\begin{proposition}
\label{prop:AI}
Index $AI$ satisfies P1, P2 and P4--P6, but not P3.
\end{proposition}

\begin{proof}
We shall prove all the properties separately.
\begin{description}
	\item[P1]: Assuming $\nu=0$, then we should prove $AI(\mathbf{A})=0 \Leftrightarrow \mathbf{A} \in \mathcal{A}^{*}$.\\
($\Rightarrow$): As all the terms of the sum in (\ref{eq:AI}) are non-negative, if $AI(\mathbf{A})=0$, then they must all be equal to zero. Such terms equal zero only when all the numerators equal zero, i.e. when $\max(r_{ij})=\min(r_{ij})~\forall i<j$, which implies that $\mathbf{A}\in\mathcal{A}^{*}$. \\
($\Leftarrow$): If $\mathbf{A} \in \mathcal{A}^{*}$, then all the elements $r_{ij}$ are singletons and therefore $\max(r_{ij})=\min(r_{ij})~\forall i<j$, implying that the numerators in (\ref{eq:AI}) equals zero and $AI(\mathbf{A})=0$
\item[P2]: Straightforward.
\item[P3]: It is sufficient to consider the following matrix $\mathbf{A}$ and its derived $\mathbf{A}(2)$ and $\mathbf{A}(3)$
\begin{equation}
\label{eq:notA3}
\mathbf{A}=
\begin{pmatrix}
1 & 2 & 8 \\
1/2 & 1 & 2 \\
1/8 & 1/2 & 1
\end{pmatrix}
\hspace{0.36cm}
\mathbf{A}(2)=
\begin{pmatrix}
1 & 2^{2} & 8^{2} \\
1/2^{2} & 1 & 2^{2} \\
1/8^{2} & 1/2^{2} & 1
\end{pmatrix}
\hspace{0.36cm}
\mathbf{A}(3)=
\begin{pmatrix}
1 & 2^{3} & 8^{3} \\
1/2^{3} & 1 & 2^{3} \\
1/8^{3} & 1/2^{3} & 1
\end{pmatrix}
\end{equation}
and observe that $I(\mathbf{A}(2))\approx 0.108$ and $I(\mathbf{A}(3))\approx 0.068$. Hence $I(\mathbf{A}(2))>I(\mathbf{A}(3))$ and P3 is not satisfied.

\item[P4]: Given $a_{12},a_{23},\ldots,a_{n-1 \, n}$, a \emph{consistent} pairwise comparison matrix of order $n$ can be equivalently written as
\begin{equation}
\mathbf{A}=
\begin{pmatrix}
1                                      & a_{12}                                & a_{12}a_{23}     & \cdots  & a_{12}\cdot \ldots \cdot a_{n-1 \, n}  \\
\frac{1}{a_{12}}                       &    1                                  &     a_{23}       & \cdots  & a_{23} \cdot \ldots \cdot a_{n-1 \, n}  \\
\cdots                                 & \cdots                                & \cdots                             & \cdots  & \cdots    \\
\frac{1}{a_{12} \cdot \ldots \cdot a_{n-2 \, n-1}} & \frac{1}{a_{23} \cdot \ldots \cdot a_{n-2 \, n-1}}& \frac{1}{a_{34}\cdot \ldots \cdot a_{n-2 \, n-1}}                                  & \cdots  & a_{n-1 \, n} \\
\frac{1}{a_{12} \cdot \ldots \cdot a_{n-1 \, n}}   & \frac{1}{a_{23} \cdot \ldots \cdot a_{n-1 \, n}}  & \frac{1}{a_{34} \cdot \ldots \cdot a_{n-1 \, n}}& \cdots  & 1
\end{pmatrix} \in \mathcal{A}^{*}
\end{equation}
or, more compactly, as $\mathbf{B}=(b_{ij})_{n\times n}$ where
\[
b_{ij}=
\begin{cases}
\prod_{p=i}^{j-1}a_{p \, p+1},   & \forall i < j \\
1,                               & \forall i = j \\
1/ \prod_{p=i}^{j-1}a_{p \, p+1},& \forall i > j
\end{cases}
\]
Then, each element of the auxiliary matrix $\mathbf{R}$ is as follows
\[
r_{ij}= \{ b_{ik} b_{kj} | k=1,\ldots,n \}~~\forall i,j .
\]
Now, to test P4, without loss of generality, we fix the pair $(1,n)$ and replace $a_{1n}$ and $a_{n1}$ with $a_{1n}^{\delta}$ and $a_{n1}^{\delta}$, respectively. Consequently, $b_{1n}$ and $b_{n1}$ are replaced by $b_{1n}^{\delta}$ and $b_{n1}^{\delta}$. Hence, for all $i<j$
\[
r_{ij}=
\begin{cases}
 \left\{ b_{ij}  \right\}, &\forall i,j \notin \{ 1,n \} \\
 \left\{ b_{ij},b_{1n}^{\delta}\frac{b_{ij}}{b_{1n}} \right\}, & \text{otherwise.}
\end{cases}
\]
Considering the definition of $AI$ we reckon that the terms associated with $r_{ij}$ for $i,j \notin \{ 1,n \}$ equals zero. Therefore, we shall prove that all the other terms are non-decreasing functions of $\delta$. We can rewrite
\[
\left\{ b_{ij},b_{1n}^{\delta}\frac{b_{ij}}{b_{1n}} \right\}= \left\{ b_{ij},b_{ij}b_{1n}^{\delta - 1} \right\}
\]
and with $x:=b_{ij},~y:=b_{1n},~\mu=\delta-1 $, it boils down to prove that
\begin{equation}
\label{eq:quantities}
\frac{\max \{x,x y^{\mu} \} - \min \{x,x y^{\mu} \}}{(1+\max \{x,x y^{\mu} \}) (1+\min \{x,x y^{\mu}\})}
\end{equation}
is a non-decreasing function of $\mu>0$ when also $x,y>0$. Now we should examine the two cases (i) $x<xy^{\mu}$ and (ii) $x>xy^{\mu}$. We start with $x<xy^{\mu}$ and, considering that
\[
x y^{\mu}>x \Leftrightarrow y^{\mu}>1 \Leftrightarrow y>1
\]
and that therefore, for the case $xz>x$, $y^{\mu}$ is always an increasing function of $\mu$. Hence, we can substitute $y^{\mu}$ with $z>1$ and (\ref{eq:quantities}) can be replaced by
\begin{equation}
\label{eq:quantities2}
\frac{\max \{x,x z \} - \min \{x,x z \}}{(1+\max \{x,x z \}) (1+\min \{x,x z\})}~~~(x>0, z>1).
\end{equation}
Considering that we are in the case with $xy>x$, we simplify (\ref{eq:quantities2}), and obtain
\begin{equation}
\phi_{(i)}=\frac{ xz-x}{(1+xz ) (1+x)}.
\end{equation}
So now we shall prove that $\frac{\partial \phi_{(i)}}{\partial z}$ is positive for all $x>0,z>1$.

\begin{align*}
\frac{\partial \phi_{(i)}}{\partial z} &= \frac{x}{(1+x)(1+xz)} - \frac{x(xz-x)}{(1+x)(1+xz)^{2}}\\
                                   ~   &= \frac{x(1+xz)-x(xz-x)}{(1+x)(1+xz)^{2}} \\
                                   ~   &= \frac{x(1+x)}{(1+x)(1+xz)^{2}} \\
                                   ~   &= \frac{x}{(1+xz)^{2}}.
\end{align*}

This last quantity is always positive for $x>0$. A very similar result can be derived for the case (ii) $x>xy$ and thus $AI$ satisfies P4.
\end{description}
It can be checked that $AI$ also satisfies properties P5 and P6. 
\end{proof}

Although in its present form index $AI$ does not satisfy P3, the underlying idea is ingenious and it is sufficient to adjust it, i.e. discard the normalization at the denominator, to make it satisfy P3.

\begin{proposition}
The inconsistency index
\[
AI^{*}(\mathbf{A})=\frac{1}{n(n-1)} \sum_{i=1}^{n} \sum_{j=1}^{n} \left( \max(r_{ij})-\min(r_{ij}) \right)
\]
satisfies properties P1--P6.
\end{proposition}

\begin{proof}
It follows from the proof of Proposition \ref{prop:AI} that $AI^{*}$ satisfies P1, P2, P4, P5, and P6. To show that P3 is satisfied, it is sufficient to take the arguments of the sum $\sum_{i=1}^{n} \sum_{j=1}^{n} \left( \max(r_{ij})-\min(r_{ij}) \right)$ and consider that they are all non-negative, since $\max(r_{ij}) \geq \min(r_{ij}) ~\forall i,j$. Consequently, the terms $\left( \max(r_{ij})^{b} - \min(r_{ij})^{b}\right) \geq 0$ are monotone non-decreasing functions with respect to $b>1$, and P3 is satisfied. 
\end{proof}

\begin{example}
Consider the pairwise comparison matrix $\mathbf{A}$ in (\ref{eq:notA3}) and its associated $\mathbf{A}(b)=\left( a_{ij}^{b} \right)_{3 \times 3}$. Figure \ref{fig:AI} contains the plots of $AI$ and $AI^{*}$ for $\mathbf{A}(b)$ as functions of $b$ and shows their different behaviors.
\begin{figure}[h!bt]
\centering
\subfloat[\footnotesize{Index $AI$ can be decreasing w.r.t $b$, and even tend to $0$, when $b \rightarrow \infty$.}]{
\includegraphics[width=0.43\textwidth]{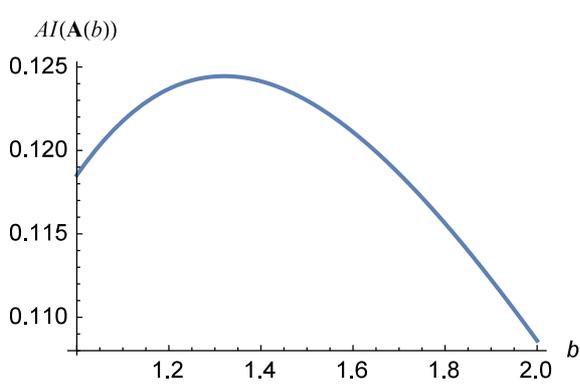} 
} \hspace{1cm}
\subfloat[\footnotesize{Index $AI^{*}$ is monotone non-decreasing w.r.t. $b$.}]{
\includegraphics[width=0.43\textwidth]{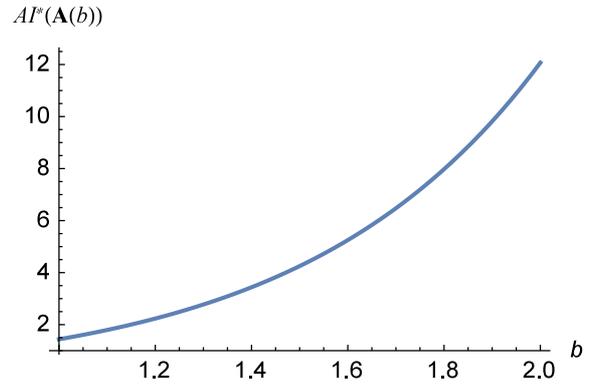}
}
\caption{Comparison between $AI$ and $AI^{*}$ with respect to P3. \label{fig:AI}}
\end{figure}
\end{example}

%
%

\subsection{Index by Wu and Xu}

\citet{WuXu2012} defined their inconsistency index using some properties of the Hadamard product of positive matrices. 
\begin{definition}[Index by \citet{WuXu2012}]
\label{def:D}
The index defined by Wu and Xu is
\[
CI_{H}(\mathbf{A})=\frac{1}{n^{2}} \sum_{i=1}^{n}\sum_{j=1}^{n} a_{ij}g_{ji} \, ,
\]
where $g_{ij}= \left( \prod_{k=1}^{n} a_{ik}a_{kj} \right)^{\frac{1}{n}}$.
\end{definition}
Note that the matrix $\mathbf{G}=\left( g_{ij} \right)_{n \times n} \in \mathcal{A}^{*}$ can be interpreted as a consistent approximation of $\mathbf{A}$. In the original paper $CI_{H}$ was used in a mathematical model to manage consistency and consensus at once. Until now, no formal or numerical analysis has been made on $CI_{H}$ and there is no information on its properties. However, with the following proposition we show that it satisfies P1--P6.
\begin{proposition}
Index $CI_{H}$ satisfies the properties P1--P6.
\end{proposition}

\begin{proof}
We shall show that P1 is satisfied, with $\nu=1$. First we need to prove that $CI_{H}(\mathbf{A})=1 \Rightarrow \mathbf{A}\in \mathcal{A}^{*}$.
\begin{equation*}
CI_{H}(\mathbf{A})= \frac{1}{n^{2}} \sum_{i=1}^{n}\sum_{j=1}^{n} a_{ij}g_{ji} = \frac{1}{n}+\frac{1}{n^{2}}\sum_{i=1}^{n-1}\sum_{j=i+1}^{n}
\underbrace{\left( a_{ij}g_{ji}+\frac{1}{a_{ij}g_{ji}} \right)}_{\psi(a_{ij},g_{ji})}
\end{equation*}
Now it can be seen that each function $\psi(a_{ij},g_{ji})$ attains its global minimum, equal to $2$, when $a_{ij}g_{ji}=1$, which is a restatement of the consistency condition. In this case, to receive the hint that $\nu = 1$, it is enough to simplify the sum,
\[
CI_{H}(\mathbf{A})
=
\frac{1}{n}+\frac{1}{n^{2}}\sum_{i=1}^{n-1}\sum_{j=i+1}^{n}
2
= \frac{1}{n}+ \frac{1}{n^{2}} \cdot \frac{n(n-1)}{2} \cdot 2 = 1
\]
Now in the other direction, $\mathbf{A}\in \mathcal{A}^{*} \Rightarrow CI_{H}(\mathbf{A})=1$, it suffices to expand $CI_{H}(\mathbf{A})$
\[
CI_{H}(\mathbf{A}) = \frac{1}{n^{2}} \sum_{i=1}^{n}\sum_{j=1}^{n} \left( a_{ij} \left( \prod_{k=1}^{n} \frac{1}{a_{ik}a_{kj}} \right)^{1/n} \right).
\]
Since consistency implies $a_{ik}a_{kj}=a_{ij}$ we have 
\[
CI_{H}(\mathbf{A}) = \frac{1}{n^{2}} \sum_{i=1}^{n}\sum_{j=1}^{n} \left( a_{ij} \frac{1}{a_{ij}} \right) = \frac{1}{n^{2}} \sum_{i=1}^{n}\sum_{j=1}^{n}  1  = 1.
\]
It is simple, and thus omitted, to show that P2, P5 and P6 hold. To prove P3, we shall call $(ag)_{ij}=a_{ij}^{b}g_{ji}^{b}$. By expanding $g_{ji}$,
\[
(ag)_{ij}=a_{ij}^{b} \left( a_{j1}^{b}a_{1i}^{b}\cdot a_{j2}^{b}a_{2i}^{b} \cdot \ldots \cdot a_{jn}^{b}a_{ni}^{b} \right)^{1/n}
=( \underbrace{a_{ij} g_{ji}}_{>0} )^{b}.
\]
Since $(ag)_{ij}=1/(ag)_{ji}$, by summing $(ag)_{ij}$ and $(ag)_{ji}$ we obtain
\[
(ag)_{ij} + (ag)_{ji} =  \left( a_{ij} g_{ji} \right)^{b} + \frac{1}{\left( a_{ij} g_{ji}  \right)^{b}} ~~\forall i,j,
\]
which is an increasing function for $b>0$. Since this holds for the general pair of indices $\{ i,j \}$, the index satisfies P3.\\
To prove P4, assume, without loss of generality, that the element to be modified is $a_{1n}$. For sake of simplicity, we can modify it and its reciprocal by multiplying them by $\beta > 0$. All the $(ag)_{ij}$ with $i,j \notin \{ 1,n\}$ will be equal to 1. For the entries with one index $i,j$ equal to either $1$ or $n$ we have
\[
(ag)_{ij}=a_{ij}(\underbrace{a_{j1}a_{1i}\cdot a_{j2}a_{2i} \cdot \ldots \cdot a_{jn}a_{ni}}_{a_{ji}^{n}} \beta)^{1/n},
\]
meaning that $g_{ji}=a_{ji}\beta^{1/n}$. As we know that $g_{ij}=1/g_{ji}$, by summing $(ag)_{ij}$ and $(ag)_{ji}$ and simplifying, one obtains
\[
\frac{1}{\beta^{1/n}} + \beta^{1/n},
\]
which is a strictly convex function for $\beta > 0$ with minimum in $\beta=1$. Similarly, for $(ag)_{1n}$ and $(ag)_{n1}$, it is
\[
(ag)_{1n} + (ag)_{n1} = \frac{1}{\beta^{2/n}} + \beta^{2/n},
\]
which shares the same property. 
\end{proof}

\subsection{Cosine Consistency Index and other indices}

Many times it is not easy to prove whether an index satisfies some properties, but numerical tests and counterexamples can always be used to show that the index does not. This was the case with the Cosine Consistency Index.
\begin{definition}[Cosine Consistency Index \citep{KouLin2014}]
\label{def:CCI}
The Cosine Consistency Index is
\[
CCI(\mathbf{A})=\sqrt{\sum_{i=1}^{n} \left( \sum_{j=1}^{n}  b_{ij} \right)^{2}} \Bigg/ n ,
\]
where $b_{ij}=a_{ij} \big/ \sqrt{\sum_{k=1}^{n} a_{kj}^{2}}$.
\end{definition}
Note that $CCI(\mathbf{A})\in [0,1]$ and its interpretation is reversed, meaning that the greater its value the \emph{less} inconsistent $\mathbf{A}$ is. It is simple, and it can also be found in the original paper, to show that $CCI$ satisfies P1, P2, P5, and P6.  For instance, in the case of P1, the proof comes directly from Equation 6 and Theorem 3 in the paper by \citet{KouLin2014}. However, the following counterexample suffices to show that $CCI$ does \emph{not} satisfy P3.
\begin{example}
Consider the matrix
\begin{equation}
\label{eq:A_CCI}
\mathbf{A}=
\begin{pmatrix}
1 & 3 & 7 \\
1/3 & 1 & 1/2 \\
1/7 & 2 & 1
\end{pmatrix}
\end{equation}
and its associated $\mathbf{A}(b)=(a_{ij}^{b})_{n \times n}$. The plot of $CCI(\mathbf{A}(b))$ is reported in Figure \ref{fig:CCI} and shows that $CCI$ does not satisfy P3.
\begin{figure}[htbp]
	\centering
		\includegraphics[width=0.43\textwidth]{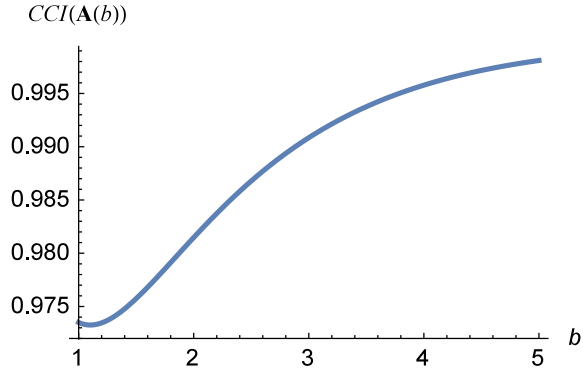}
	\caption{For the matrix $\mathbf{A}$ in (\ref{eq:A_CCI}), an intensification of preferences decreases the inconsistency.}
	\label{fig:CCI}
\end{figure}
\end{example}

Often, although in their present forms they do not satisfy P1--P6, ideas behind indices are valid and slight modifications are sufficient to make them satisfy a set of properties. One example is the index $NI_{n}^{\sigma}$ proposed by \citet{RamikKorviny2010} which was later studied \citep{Brunelli2011}. Another concrete example is the Relative Error index by \citet{Barzilai1998} which does not satisfy P4 and P5.
In its original formulation such index is
\[
RE(\mathbf{A})=\frac{\sum_{i=1}^{n}\sum_{j=1}^{n}\left( p_{ij} - d_{i} + d_{j} \right)^{2}}{ \sum_{i=1}^{n}\sum_{j=1}^{n} \left(p_{ij}\right)^{2}},
\]
where $p_{ij}=\log a_{ij}$ and $d_{i}=\frac{1}{n} \sum_{k=1}^{n} p_{ik}$, and where the denominator acts as a normalization factor.
Here it can be proved that, if we discard the denominator, we obtain
\[
RE^{*}(\mathbf{A})=\sum_{i=1}^{n}\sum_{j=1}^{n}\left( p_{ij} - d_{i} + d_{j} \right)^{2},
\]
which, unlike $RE$, satisfies all the properties.
\begin{proposition}
Index $RE^{*}$ satisfies all the properties P1--P6.
\end{proposition}

\begin{proof}
It is easy to show, and therefore omitted, that P1, P2, P5 and P6 are satisfied. We shall prove P3 and P4 separately.
To prove the satisfaction of P3 we consider  $\mathbf{A}(b) =\left( a_{ij}^{b} \right)$ and note that applying the logarithmic transformation to its entries, we obtain $ \left( b \log a_{ij} \right)_{n \times n} = \left( p_{ij}  \cdot b \right)_{n \times n}$. Hence,
\begin{equation*}
RE^{*}(\mathbf{A}(b))=\sum_{i=1}^{n}\sum_{j=1}^{n}\left( b \cdot p_{ij} - \frac{1}{n} \sum_{k=1}^{n} b \cdot p_{ik} + \frac{1}{n} \sum_{k=1}^{n} b \cdot p_{jk} \right)^{2} = b^{2} \cdot RE^{*}(\mathbf{A}),
\end{equation*}
which implies that $RE^{*}(\mathbf{A}(b))$ is monotone non-decreasing for $b>1$.\\
To show that P4 holds, let us consider the matrix $\mathbf{A}\in \mathcal{A}^{*}$, and its associated $\mathbf{P}=(p_{ij})_{n \times n} = (\log a_{ij})_{n \times n}$. P4 can equivalently be restated as the property that, if we take an entry $p_{pq}$ and its reciprocal $p_{qp}$ and substitute them with $p_{pq}+ \xi$ and $p_{qp} - \xi$, respectively, then the inconsistency index $RE^{*}$ is a quasi-convex function of $\xi$ with minimum in $\xi=0$. From the proof of Proposition 5 in \citep{BrunelliFedrizziAxioms} one recovers that, by introducing $\xi$, it is
\[
\sum_{i=1}^{n}\sum_{j=1}^{n}\left( p_{ij} - d_{i} + d_{j} \right)^{2} = 4(n-2)\left( \frac{\xi}{n}\right)^{2}+2\left( \frac{n-2}{n} \xi \right)^{2}.
\]
Thus one obtains,
\[
RE^{*}(\mathbf{A}_{pq}(\xi))=4(n-2)\left( \frac{\xi}{n}\right)^{2}+2\left( \frac{n-2}{n} \xi \right)^{2}
                    ={\frac{2(-2+n)\xi^{2}}{n}}
										=  \xi ^{2} \underbrace{\frac{2(n-2)}{n}}_{>0}
\]
which is a decreasing function of $\xi$ for $\xi <0$ and an increasing function for $\xi>0$. 
\end{proof}

\section{Discussion and conclusions}
\label{sec:5}


Choosing the most suitable inconsistency index is of considerable importance, 
yet formal studies had not been undertaken until very recently  \citep{BrunelliFedrizziAxioms,KoczkodajSzwarc2014}. This is in contrast with the existence of long-standing studies on other aspects of pairwise comparisons. One of these is the choice of the method for deriving the priority vector, for which axiomatic studies have been proposed in the literature already in the Eighties \citep{CookKress1988,Fichtner1986} and in the Nineties \citep{Barzilai1997}. Nevertheless, it has been shown by numerical studies \citep{IshizakaLusti2006} that, excepts for some particular cases, such differences can be negligible and that therefore, in most of the cases, choosing one method or another does not really influence the final outcome.


In light of the recently proposed five properties for inconsistency indices, the contribution of this research is at least twofold:
\begin{itemize}
	\item Firstly, it introduces and justifies a sixth property (P6) and shows that, together with the other five, it forms an an independent and logically consistent set of properties.
\item Secondly, the paper further analyzes the satisfaction of the properties P1--P6. Four inconsistency indices have been considered from the literature and it was found that two of them fail to fully satisfy the set of properties P1--P6. A simple adjustment of one of these indices was proposed to make it fit P1--P6.
\end{itemize}

Table \ref{tab:summary} presents a summary of the findings of this research and shows how they expanded the original set of properties for inconsistency indices \citep{BrunelliFedrizziAxioms}. It is remarkable that, in the form in which they were originally introduced in the literature, the majority of the indices satisfy only some of them. This seems to indicate that the definition of the properties and the analysis of their satisfaction is \emph{not} a mere theoretical exercise.\\
\begin{table}[htbp]
   
    \centering
        \begin{tabular}{l|cccccc}
            ~                                   & P1 & P2 & P3 & P4 & P5 & \multicolumn{1}{:c}{P6} \\
            \hline
      $CI$                   &  \cmark & \cmark  & \cmark  & \cmark  &  \cmark & \multicolumn{1}{:c}{\textbf{\cmark}} \\
      $GW$                  &  \cmark & \cmark  & \xmark  & ---  &  \cmark & \multicolumn{1}{:c}{\textbf{\cmark}}  \\
      $GCI$                &  \cmark & \cmark  & \cmark  & \cmark  &  \cmark & \multicolumn{1}{:c}{\textbf{\cmark}}  \\
      $RE$                   &  \cmark & \cmark  & \cmark  & \xmark  &  \xmark & \multicolumn{1}{:c}{\textbf{\cmark}}  \\
      $CI^{*}$             &  \cmark & \cmark  & \cmark  & \cmark  &  \cmark & \multicolumn{1}{:c}{\textbf{\cmark}}  \\
      $HCI$                &  \cmark & \cmark  & \xmark  & \cmark  &  \cmark & \multicolumn{1}{:c}{\textbf{\cmark}}  \\
      $NI^{\sigma}_{n}$      &  \cmark & \cmark  & ---  & \xmark  &  \cmark & \multicolumn{1}{:c}{\textbf{\cmark}}  \\ \cdashline{1-6}

			$K$ (Definition \ref{def:K})     							&  \textbf{\cmark} & \textbf{\cmark}  & \textbf{\cmark}  & \textbf{\cmark}  &  \textbf{\cmark} & \textbf{\cmark} \\
			$AI$ (Definition \ref{def:AI})     							&  \textbf{\cmark} & \textbf{\cmark}  & \textbf{\xmark}  & \textbf{\cmark}  &  \textbf{\cmark} & \textbf{\cmark} \\
			$CI_{H}$ (Definition \ref{def:D})     							&  \textbf{\cmark} & \textbf{\cmark}  & \textbf{\cmark}  & \textbf{\cmark}  &  \textbf{\cmark} & \textbf{\cmark} \\
			$CCI$ (Definition \ref{def:CCI})     							&  \textbf{\cmark} & \textbf{\cmark}  & \textbf{\xmark}  & ---  &  \textbf{\cmark} & \textbf{\cmark} \\
        \end{tabular}
    \caption{Summary of propositions: \cmark=property is satisfied, \xmark=property is not satisfied, \mbox{`---'=unknown}. The original results presented in this research are separated from previous ones \citep{BrunelliFedrizziAxioms} by the dashed lines.}
    \label{tab:summary}
\end{table}

The properties were here, and in previous research \citep{BrunelliFedrizziAxioms}, justified. Nevertheless, clearly, this should not prevent anyone from criticizing and improving them: it is indeed desirable that a set of properties be openly discussed within a community. In this direction, if the system P1--P6 is considered too restrictive, it is worth noting that Theorem 1 implies that any subset of the properties P1--P6 also forms an independent and logically consistent set or properties (just a more relaxed one) which, indeed, can be used for the same purposes of P1--P6.
In conclusion, it is the author's belief that a systematic study of inconsistency and inconsistency indices may bring new insights and more formal order into the evergreen topic of rational decision making. Furthermore, in the future, it should be possible to extend the set of properties to other types of numerical representations of preferences as, for instance, reciprocal preference relations \citep{Tanino1984} and skew-symmetric additive representations \citep{Fishburn1999}.

\subsubsection*{Acknowledgements}

{\footnotesize 
The author is grateful to the reviewers and the Associate Editor for their precious comments. The manuscript benefited from the author's discussions with Michele Fedrizzi and Ragnar Freij. A special mention goes to S\'{a}ndor Boz\'{o}ki who also had the intuition that a property might have been missing. This research has been financially supported by the Academy of Finland.
}

\setlength{\bibsep}{1pt plus 0.3ex}

%
%
%

\begin{thebibliography}{99}
 {\small

\bibitem[Aguar{\'o}n and Moreno-Jim{\'e}nez, 2003]{AguaronMoreno2003}
Aguar{\'o}n, J., \& Moreno-Jim{\'e}nez, J.~M. (2003).
\newblock The geometric consistency index: Approximated thresholds.
\newblock {\em European Journal of Operational Research}, 147(1), 137--145.

\bibitem[Barzilai, 1997]{Barzilai1997}
Barzilai, J. (1997).
\newblock Deriving weights from pairwise comparison matrices.
\newblock {\em The Journal of the Operational Research Society},
 48(12), 1226--1232.
	
\bibitem[Barzilai, 1998]{Barzilai1998}
Barzilai, J. (1998).
\newblock Consistency measures for pairwise comparison matrices.
\newblock {\em Journal of Multi-Criteria Decision Analysis},
7(3), 123--132.

\bibitem[Belton and Gear, 1983]{BeltonGear1983}
Belton, V., \& Gear, T., (1983).
\newblock On a short-coming of Saaty's method of analytic hierarchies.
\newblock {\em Omega},
11(3), 228--230.

\bibitem[Boz{\'o}ki and Rapcs{\'a}k, 2008]{BozokikRapcsak2008}
Boz{\'o}ki, S., \& Rapcs{\'a}k, T. (2008).
\newblock On {S}aaty's and {K}oczkodaj's inconsistencies of pairwise comparison
  matrices.
\newblock {\em Journal of Global Optimization}, 42(2), 157--175.

\bibitem[Boz{\'o}ki et al, 2013]{BozokiEtAl2013}
Boz{\'o}ki, S., Dezs{\H{o}}, L., Poesz, A., \& Temesi, J. (2013).
\newblock Analysis of pairwise comparison matrices: an empirical research.
\newblock {\em Annals of Operations Research}, 211(1), 511--528.

\bibitem[Brunelli, 2011]{Brunelli2011}
Brunelli, M. (2011).
\newblock A note on the article ``Inconsistency of pair-wise comparison matrix
               with fuzzy elements based on geometric mean'' [Fuzzy Sets and Systems
               161 {(2010)} 1604-1613].
\newblock {\em Fuzzy Sets and Systems}, 176(1), 76--78.

\bibitem[Brunelli et al, 2013a]{BrunelliCanalFedrizzi}
Brunelli, M., Canal, L., \& Fedrizzi M. (2013a).
\newblock Inconsistency indices for pairwise comparison matrices: a numerical
  study.
\newblock {\em Annals of Operations Research}, 211(1), 493--509.

\bibitem[Brunelli et al., 2013b]{BrunelliCritchFedrizzi2013}
Brunelli, M., Critch, A., \& Fedrizzi M. (2013b).
\newblock A note on the proportionality between some consistency indices in the
  {AHP}.
\newblock {\em Applied Mathematics and Computation}, 219(14), 7901--7906.

\bibitem[Brunelli and Fedrizzi, 2015a]{BrunelliFedrizziAxioms}
Brunelli, M., \& Fedrizzi, M. (2015a). 
\newblock Axiomatic properties of inconsistency indices for pairwise comparisons.
\newblock {\em Journal of the Operational Research Society}, 66(1), 1--15.

\bibitem[Brunelli and Fedrizzi, 2015b]{BrunelliFedrizzi2014}
Brunelli, M., \& Fedrizzi, M. (2015b).
\newblock Boundary properties of the inconsistency of pairwise comparisons in
  group decisions.
\newblock {\em European Journal of Operational Research}, 230(3), 765--773.

\bibitem[Cavallo and D'Apuzzo, 2009]{CavalloDApuzzo2009}
Cavallo, B., \& D'Apuzzo, L. (2009).
\newblock A general unified framework for pairwise comparison matrices in multicriterial methods.
\newblock {\em International Journal of Intelligent Systems}, 24(4), 377--398.

\bibitem[Cavallo and D'Apuzzo, 2012]{CavalloDApuzzoIPMU}
Cavallo, B., \& D'Apuzzo, L. (2012).
\newblock Investigating properties of the $\odot$-consistency index.
\newblock In: {Advances in Computational Intelligence. Communications in Computer
and Information Science}, Vol. 4, pp. 315--327.

\bibitem[Chen et al., 2015]{ChenEtAl}
Chen, K., Kou, G., Tarn, J.M., \& Song, Y. (2015).
\newblock Bridging the gap between missing and inconsistent values in eliciting preference from pairwise comparison matrices.
\newblock{\em Annals of Operations Research}, 235(1), 155--175.

\bibitem[Cook and Kress, 1988]{CookKress1988}
Cook, W.~D., \& Kress, M (1988).
\newblock Deriving weights from pairwise comparison ratio matrices: An
  axiomatic approach.
\newblock {\em European Journal of Operational Research}, 37(3), 355--362.

\bibitem[Duszak and Koczkodaj, 1994]{Koczkodaj1994}
Duszak, Z., \& Koczkodaj, W.~W. (1994).
\newblock Generalization of a new definition of consistency for pairwise
  comparisons.
\newblock {\em Information Processing Letters}, 52(5), 273--276.

\bibitem[Dyer, 1990a]{Dyer1990a}
Dyer, J.~S. (1990a).
\newblock Remarks on the analytic hierarchy process.
\newblock {\em Management Science}, 36(3), 249--258.

\bibitem[Dyer, 1990b]{Dyer1990b}
Dyer, J.~S. (1990b).
\newblock A clarification of ``Remarks on the analytic hierarchy process''.
\newblock {\em Management Science}, 36(3), 274--275.

\bibitem[Ergu et al., 2011]{ErguEtAl2011}
Ergu, D., Kou, G, Peng, Y, \& Shi, Y. (2011).
\newblock A simple method to improve the consistency ratio of the pair-wise
  comparison matrix in {ANP}.
\newblock {\em European Journal of Operational Research}, 213(1), 246--259.

\bibitem[Fichtner, 1986]{Fichtner1986}
Fichtner, J. (1986).
\newblock On deriving priority vectors from matrices of pairwise comparisons.
\newblock {\em Socio-Economic Planning Sciences}, 20(6), 341--345.

\bibitem[Fishburn, 1968]{Fishburn1968}
Fishburn, P.~C. (1968).
\newblock Utility theory.
\newblock {\em Management Science}, 14(5), 335--378.

\bibitem[Fishburn, 1999]{Fishburn1999}
Fishburn, P.~C. (1999).
\newblock Preference relations and their numerical representations.
\newblock {\em Theoretical Computer Science}, 217(2), 359--383.

\bibitem[Gass, 2005]{Gass2005}
Gass, S.~I. (2005).
\newblock Model world: The great debate --- {MAUT} versus {AHP}.
\newblock {\em Interfaces}, 35(4), 308--312.

\bibitem[H{\"a}m{\"a}l{\"a}inen and P{\"o}yh{\"o}nen, 1996]{HamalainenPoyhonen1996}
H{\"a}m{\"a}l{\"a}inen, R.~P., \& P{\"o}yh{\"o}nen, M. (1996).
\newblock On-line group decision support by preference programming in traffic
  planning.
\newblock {\em Group Decision and Negotiation}, 5(4--6), 485--500.

\bibitem[Herman and Koczkodaj, 1996]{HermanKoczkodaj1996}
Herman, M.~W., Koczkodaj, W.~W. (1996).
\newblock A Monte Carlo study of pairwise comparison.
\newblock {\em Information Processing Letters}, 57(1), 25--29.

\bibitem[Irwin, 1958]{Irwin1958}
Irwin, F.~W. (1958).
\newblock An analysis of the concepts of discrimination and preference.
\newblock {\em The American Journal of Psychology}, 71(1), 152--163.

\bibitem[Ishizaka and Lusti, 2006]{IshizakaLusti2006}
Ishizaka, A., \& Lusti, M. (2006).
\newblock How to derive priorities in {AHP}: a comparative study.
\newblock {\em Central European Journal of Operations Research},
  14(4), 387--400.

\bibitem[Kakiashvili et al., 2012]{KakiashviliEtAl2012}
Kakiashvili, T., Koczkodaj, W.~W., \& Woodbury-Smith, M. (2012).
\newblock Improving the medical scale predictability by the pairwise
  comparisons method: Evidence from a clinical data study.
\newblock {\em Computer Methods and Programs in Biomedicine}, 105(3):210--216.

\bibitem[Koczkodaj and Szwarc, 2014]{KoczkodajSzwarc2014}
Koczkodaj, W., \& Szwarc, R. (2014).
\newblock On axiomatization of inconsistency indicators in pairwise
  comparisons.
\newblock {\em Fundamenta Informaticae}, 132(4), 485--500.

\bibitem[Koczkodaj, 1993]{Koczkodaj1993}
Koczkodaj, W.~W. (1993).
\newblock A new definition of consistency of pairwise comparisons.
\newblock {\em Mathematical and Computer Modelling}, 18(7), 79--84.

\bibitem[Koczkodaj et al., 1999]{Koczkodaj1999}
Koczkodaj, W.~W., Herman, M.~W., \& Orlowski, M. (1999).
\newblock Managing null entries in pairwise comparisons.
\newblock {\em Knowledge and Information Systems}, 1(1), 119--125.

\bibitem[Koczkodaj et al., 2014]{KoczkodajEtAl2014}
Koczkodaj, W.~W., Kulakowski, K., \& Ligeza, A. (2014).
\newblock On the quality evaluation of scientific entities in {P}oland
  supported by consistency-driven pairwise comparisons method.
\newblock {\em Scientometrics}, 99(3), 911--926.


\bibitem[Kou and Liu, 2014]{KouLin2014}
Kou, G., \& Lin, C. (2014).
\newblock A cosine maximization method for the priority vector derivation in
  {AHP}.
\newblock {\em European Journal of Operational Research}, 235(1), 225--232.
	
\bibitem[Ku\l akowski, 2015]{Kulakowski2015}
Ku\l akowski, K. (2015).
\newblock Notes on order preservation and consistency in {AHP}.
\newblock {\em European Journal of Operational Research}, 245(1), 333--337.

\bibitem[Lamata and Pelaez, 2002]{LamataPelaez2002}
Lamata, M.~T., \& Pel{\'a}ez, J.~I. (2002).
\newblock A method for improving the consistency of judgements.
\newblock {\em International Journal of Uncertainty, Fuzziness and
  Knowledge-Based Systems}, 10(6), 677--686.
	
\bibitem[Lin et al., 2013]{LinEtAl2013}
Lin, C., Kou, G., \& Daji, E. (2013).
\newblock An improved statistical approach for consistency test in AHP.
\newblock {\em Annals of Operations Research}, 211(1), 289--299.
	
\bibitem[Luce and Suppes, 1965]{LuceSuppes1965}
Luce, R.~D., \& Suppes, P. (1965).
\newblock Preference, utility and subjective probability.
\newblock In: R.D. Luce, R.R. Bush, \ E.H. Galanter (Eds.), {\em Handbook of Mathematical Psychology}, pp 249--410.

\bibitem[Luce and Raiffa, 1957]{LuceRaiffa1957}
Luce, R.~D., \& Raiffa, H. (1957).
\newblock {\em Games and Decisions}.
\newblock John Wiley and Sons.

\bibitem[Maleki and Zahir, 2013]{MalekiZahir2013}
Maleki, H., \and Zahir, S. (2013).
\newblock A comprehensive literature review of the rank reversal phenomenon in
  the analytic hierarchy process.
\newblock {\em Journal of Multi-Criteria Decision Analysis}, 20(3-4), 141--155.

\bibitem[Mustajoki and H\"{a}m\"{a}l\"{a}inen, 2000]{Mustajoki2000}
Mustajoki, J., \& H\"{a}m\"{a}l\"{a}inen, R.~P. (2000).
\newblock Web-{HIPRE}: global decision support by value tree and {AHP}
  analysis.
\newblock {\em INFOR Journal}, 38(3), 208--220.

\bibitem[Nikou and Mezei, 2013]{NikouMezei2013}
Nikou, S., \& Mezei, J. (2013).
\newblock Evaluation of mobile services and substantial adoption factors with Analytic Hierarchy Process ({AHP})
  analysis.
\newblock {\em Telecommunications Policy}, 37(10), 915--929.

\bibitem[Nikou et al., 2015]{NikouMezeiSarlin2015}
Nikou, S., Mezei, J., \& Sarlin, P. (2015).
\newblock A process view to evaluate and understand preference elicitation.
\newblock {\em Journal of Multi-Criteria Decision Analysis}, 22(5--6), 305--329.

\bibitem[Pereira and Costa, 2015]{PereiraCosta2015}
Pereira, V., \& Costa, H.~G. (2015).
\newblock Nonlinear programming applied to the reduction of inconsistency in the AHP method.
\newblock {\em Annals of Operations Research}, 229(1), 635--655.

\bibitem[Ram\'{i}k and Korviny, 2010]{RamikKorviny2010}
Ram\'{i}k, J., \& Korviny, P. (2010).
\newblock Inconsistency of pair-wise comparison matrix with fuzzy elements based on geometric mean.
\newblock {\em Fuzzy Sets and Systems}, 161(11), 1604--1613.

\bibitem[Saaty, 1993]{Saaty1993}
Saaty, T.~L. (1993).
\newblock What is relative measurement? {T}he ratio scale phantom.
\newblock {\em Mathematical and Computer Modelling}, 17(4), 1--12.

\bibitem[Saaty, 2013]{Saaty2013}
Saaty, T.~L. (2013).
\newblock The modern science of multicriteria decision making and its practical
  applications: The {AHP}/{ANP} approach.
\newblock {\em Operations Research}, 61(5), 1101--1118.

\bibitem[Salo and H\"{a}m\"{a}l\"{a}inen, 1995]{SaloHamalainen1995}
Salo, A.~A., \& H\"{a}m\"{a}l\"{a}inen, R.~P. (1995).
\newblock Preference programming through approximate ratio comparisons.
\newblock {\em European Journal of Operational Research}, 82(3), 458--475.

\bibitem[Salo and H\"{a}m\"{a}l\"{a}inen, 1997]{SaloHamalainen1997}
Salo, A.~A., \& H\"{a}m\"{a}l\"{a}inen, R.~P. (1997).
\newblock On the measurement of preferences in the analytic hierarchy process.
\newblock {\em Journal of Multi-Criteria Decision Analysis}, 6(6), 309--319.


\bibitem[Shiraishi et al., 1999]{ShiraishiEtAl1999}
Shiraishi, S., Obata, T., Daigo, M., \& Nakajima, N. (1999).
\newblock Assessment for an incomplete comparison matrix and improvement of an
  inconsistent comparison: computational experiments.
\newblock In {\em ISAHP 1999}.

\bibitem[Stein and Mizzi, 2007]{SteinMizzi2007}
Stein, W.~E., \& Mizzi, P.~J. (2007).
\newblock The harmonic consistency index for the analytic hierarchy process.
\newblock {\em European Journal of Operational Research}, 177(1):488--497,
  2007.

\bibitem[Tanino, 1984]{Tanino1984}
Tanino, T. (1984).
\newblock Fuzzy preference orderings in group decision making.
\newblock {\em Fuzzy Sets and Systems}, 12(2), 117--131.


\bibitem[Watson and Freeling, 1982]{WatsonFreeling1982}
Watson, S.~R., \& Freeling, A.~N.~S. (1982).
\newblock Assessing attribute weights.
\newblock {\em Omega}, 10(6), 582--583.

\bibitem[Watson and Freeling, 1983]{WatsonFreeling1983}
Watson, S.~R., \& Freeling, A.~N.~S. (1983).
\newblock Comment on: assessing attribute weights by ratios.
\newblock {\em Omega}, 11(1), 13.

\bibitem[Wu and Xu, 2012]{WuXu2012}
Wu, Z., \& Xu, J. (2012).
\newblock A consistency and consensus based decision support model for group
  decision making with multiplicative preference relations.
\newblock {\em Decision Support Systems}, 52(3), 757--767.
}

\end{thebibliography}
%
%
\section*{Appendix: Proof of Theorem 1}
To prove \emph{logical consistency}, it is sufficient to find an instance of $I:\mathcal{A}\rightarrow \mathbb{R}$ which satisfies all the properties P1--P6. One such instance is the following function
\begin{equation}
\label{eq:I*}
I^{*}(\mathbf{A})=\sum_{i=1}^{n-2}\sum_{j=i+1}^{n-1}\sum_{k=j+1}^{n} \left( \frac{a_{ik}}{a_{ij}a_{jk}} + \frac{a_{ij}a_{jk}}{a_{ik}} - 2 \right)
\end{equation}
To prove the \emph{independence} of P1--P6, it is sufficient to find a function satisfying all properties except one, for all the properties. The examples of inconsistency indices proposed by \citet{BrunelliFedrizziAxioms} to prove the independence of the system P1--P5 are invariant under transposition. If follows that P1--P5 are logically independent within the system P1--P6. It remains to show that P6 does \emph{not} depend on P1--P5.
 Consider that, if $\mathbf{A}$ has one row, say $H$, whose non-diagonal elements are all greater than one, i.e. $a_{Hj}>1~\forall j\neq H$, then this property is shared by any matrix $\mathbf{P}\mathbf{A}\mathbf{P}^{T}$, where $\mathbf{P}$ is any permutation matrix, but not by its transpose $\mathbf{A}^{T}$.
Taking into account the inconsistency index $I^{*}$ in (\ref{eq:I*}), and defining $H$ as the row with the greatest non-diagonal element, then the function
\begin{equation}
\label{eq:ex_A6}
I_{\lnot 6}(\mathbf{A})=I^{*}(\mathbf{A}) \cdot \underbrace{\left( 1 + \max \left\{ \min_{j\neq H} \{ a_{Hj} - 1 \} ,0 \right\} \right)}_{M}
\end{equation}
is invariant under row-column permutation but not under transposition. Hence, $I_{\lnot 6}$ satisfies AP but not P6. To prove the independence of P6, it remains to show that (\ref{eq:ex_A6}) satisfies P1 and P3--P5. It is easy, and thus omitted, to show that P1, P3, and P5 are satisfied. To prove it for P4, we note that any $\mathbf{A}\in \mathcal{A}^{*}$ can be rewritten as
\begin{equation}
\mathbf{A}=
\begin{pmatrix}
1                                      & a_{12}                                & a_{12}a_{23}     & \cdots  & a_{12} \cdot \ldots \cdot a_{n-1 \, n}  \\
\frac{1}{a_{12}}                       &    1                                  &     a_{23}       & \cdots  & a_{23} \cdot \ldots \cdot a_{n-1 \, n}  \\
\cdots                                 & \cdots                                & \cdots                             & \cdots  & \cdots    \\
\frac{1}{a_{12} \cdot \ldots \cdot a_{n-2 \, n-1}} & \frac{1}{a_{23} \cdot \ldots \cdot a_{n-2 \, n-1}}& \frac{1}{a_{34} \cdot \ldots \cdot a_{n-2 \, n-1}}                                  & \cdots  & a_{n-1 \, n} \\
\frac{1}{a_{12} \cdot \ldots \cdot a_{n-1 \, n}}   & \frac{1}{a_{23} \cdot \ldots \cdot a_{n-1 \, n}}  & \frac{1}{a_{34} \cdot \ldots \cdot a_{n-1 \, n}}& \cdots  & 1
\end{pmatrix} \in \mathcal{A}^{*}
\end{equation}
Without loss of generality let us consider $a_{1n}$ and its reciprocal $a_{n1}$ and replace them with $a_{1n}^{\delta}$ and $a_{n1}^{\delta}$, respectively.
Then, by calling $\mathbf{A}_{1n}^{\delta}$ the new matrix and bearing in mind that $\mathbf{A} \in \mathcal{A}^{*}$, we have
\begin{align*}
I^{*}(\mathbf{A}_{1n}^{\delta}) &= \sum_{j=2}^{n-1} \left( \frac{a_{1n}^{\delta}}{a_{1j}a_{jn}} + \frac{{a_{1j}a_{jn}}}{{a_{1n}^{\delta}}} - 2 \right)\\  &= (n-2) \left( \frac{ \left( a_{12} \cdot \ldots \cdot a_{n-1 \, n} \right)^{\delta}}{a_{12} \cdot \ldots \cdot a_{n-1 \, n}} + \frac{a_{12} \cdot \ldots \cdot a_{n-1 \, n}}{ \left( a_{12} \cdot \ldots \cdot a_{n-1 \, n} \right)^{\delta}} -2 \right)
\end{align*}
If $H \notin \{ 1,n \}$, then, in (\ref{eq:ex_A6}) $M$ is constant and P4 holds in this case. Also if $H \in \{1,n \}$ and $\min_{j \neq H} \{ a_{Hj} \} \neq a_{1n}$, then $M$ is constant and P4 is satisfied. Finally, if $H=1$ and $\min_{j \neq H} \{ a_{Hj} \} = a_{1n}$, it is
\begin{equation}
\label{eq2}
I_{\lnot 6}(\mathbf{A}_{1n}^{\delta})= I^{*}(\mathbf{A}_{1n}^{\delta}) \cdot \left( 1 +  a_{1n}^{\delta} - 1 \right)= I^{*}(\mathbf{A}_{1n}^{\delta}) \cdot   a_{1n}^{\delta} 
\end{equation}
which can be reduced to
\begin{align*}
I_{\lnot 6}(\mathbf{A}_{1n}^{\delta})  =& \overbrace{(n-2) \left( \frac{ \left( a_{12} \cdot \ldots \cdot a_{n-1 \, n} \right)^{\delta}}{a_{12} \cdot \ldots \cdot a_{n-1 \, n}} + \frac{a_{12} \cdot \ldots \cdot a_{n-1 \, n}}{ \left( a_{12} \cdot \ldots \cdot a_{n-1 \, n} \right)^{\delta}} -2 \right)}^{I^{*}(\mathbf{A}^{\delta}_{1n})}  \cdot \overbrace{(a_{12} \cdots a_{n-1 \, n})^{\delta}}^{a_{1n}^{\delta}}
 \\
= & (n-2) \left(\frac{a_{1n}^{2 \delta}}{a_{1n}} + a_{1n} - 2 a_{1n}^{\delta}    \right).
\end{align*}
Considering that, from $\mathbf{A}\in \mathcal{A}^{*}$ and $H=1$, it follows that $a_{1n} \geq 1$ and the partial derivative in $\delta$ is
\begin{align*}
\frac{\partial I_{\lnot 6}(\mathbf{A}_{1n}^{\delta})}{\partial \delta}= & (n-2) \left( 2a_{1n}^{2\delta - 1}\log(a_{1n}) - 2 a_{1n}^{\delta} \log(a_{1n}) \right) \\
=& \underbrace{(n-2)}_{>0} \underbrace{\left( 2 a_{1n}^{\delta - 1}\right)}_{>0} \left( a_{1n}^{\delta}-a_{1n} \right) \underbrace{\left( \log a_{1n} \right)}_{>0},                                      
\end{align*}
which is always non-negative for $\delta >1$ and non-positive for $ 0 < \delta < 1$. Thus, P4 is satisfied and the properties P1--P6 are logically independent.

\end{document}